\documentclass[graybox]{svmult}

\usepackage{makeidx}         
\usepackage{graphicx}   

\makeindex             

\usepackage{stmaryrd}
\usepackage{times}
\usepackage{url}
\usepackage{latexsym}
\usepackage{amsmath}
\usepackage{amssymb}  
\usepackage{latexsym} 

\usepackage[all]{xy}

\newcommand{\semantics}[1]{[\![ #1 ]\!]}

\newcommand{\Rel}{\mathbf{Rel}}
\newcommand{\FdVect}{\mathbf{FdVect}}

\newcommand{\cat}[1]{\mathbf{#1}}
\newcommand{\cC}{\cat{C}}

\newcommand{\VRel}{\mathbf{V}\mbox{-}\mathbf{Rel}}
\newcommand{\V}{\mathbf{V}}
\newcommand{\e}{\mathsf{e}}
\newcommand{\tensor}{\bullet}
\newcommand{\id}{id}

\newcommand{\relto}{\nrightarrow}

\begin{document}

\title*{Fuzzy Generalised Quantifiers  for Natural Language in  Categorical Compositional Distributional Semantics}
\author{M\v{a}tej Dost\'{a}l,  Mehrnoosh Sadrzadeh,  Gijs Wijnholds}

\institute{M\v{a}tej Dost\'{a}l \at Czech Technical University,  Prague \email{dostamat@fel.cvut.cz}
\and Mehrnoosh Sadrzadeh \at University College London, UK \email{m.sadrzadeh@ucl.ac.uk}
\and Gijs Wijnholds \at Queen Mary University of London , UK \email{g.j.wijnholds@qmul.ac.uk}}
%
%
\maketitle

\abstract{
Recent work on compositional distributional models shows that bialgebras over finite dimensional vector spaces can be applied to treat generalised quantifiers for natural language. That technique requires one to construct the vector space over powersets, and therefore is computationally costly. In this paper, we overcome this problem by considering fuzzy versions of quantifiers along the lines of Zadeh, within the category of many valued relations. We show that this category is a concrete instantiation of the compositional distributional model. We show that the semantics obtained in this model is equivalent to the semantics of the fuzzy quantifiers  of Zadeh. As a result,  we are now able to treat fuzzy quantification without requiring a powerset construction. 
%
}

\section{Introduction}
\label{sec:Intro}

Distributional semantics is inspired by the idea of Firth that   words can be represented by the company they keep \cite{Firth}. This idea was  formalised by computational linguistics and by information retrieval researchers;  they represented words by the  contexts in which they often occurred. Rubenstein and Goodenough introduced the concept  of  a co-occurrence matrix: a matrix whose columns are context words, whose rows are target words, and whose entries are [a function of] the number of times the target and context words occurred together in a window of a fixed size \cite{Rubenstein}. Later Salton, Wong and Yang employed similar ideas to index words in a document \cite{Salton}. In either setting,  a word is represented by a vector:  the row associated to it in a co-occurrence matrix, which is a  vector in the vector space spanned by the context words.  This representation has been applied  to  natural language tasks such as word similarity, disambiguation, and entailment \cite{Rubenstein,Curran,Turney,Schutze,Weeds,Geffet,Kotlerman} and information retrieval tasks, such as  clustering, indexing, and search \cite{Salton,Landauer,lin1998automatic}. 

A challenge to distributional semantics was that its underlying hypothesis did make sense for words, but no longer for  complex language units such as sentences. In an attempt to extend distributional semantics from words to sentences, Clark and Pulman put forward the idea of tracing the parse tree of a sentence,  forming the tensor product of words and their grammatical roles, and representing the sentence by the resulting vector \cite{ClarkPulman}. Due to the high dimensionality of the resulting space (in which the vector of the sentence lives), this idea itself did not lead to tangible applications and experimental results in language tasks. It was however followed up by a series of related work, referred to by \emph{compositional distributional semantics},  based on the principle of compositionality of Frege, that the meaning of a sentence is a function of the meanings of its parts. The approaches within this field,  mapped the grammatical structure of  sentences to  linear maps that acted  on the representations of the words therein, for example see the work of  \cite{BaroniZam,Maillard2014,Krishnamurthy2013,Steedman2013,Coeckeetal}. These models provided  concrete vector constructions and were tested on natural language tasks such as sentence similarity, disambiguation, and entailment 
\cite{baroni-etal-2014,MultiStep,GrefenSadrCL,KartSadr,KartSadrLACL,Bankovaetal}. They were however mostly focused on elementary fragments of language and left the treatment of  logical operations such as conjunction, disjunction, and quantification to further work. Recently,  quantification was tackled  in  \cite{HedgesSadr,Sadrzadeh16} based on  a model that sits within the setting of \cite{Coeckeetal}. This setting is based on theory of  compact closed categories \cite{KellyLaplaza80}. It was shown that the bialgebras over these categories  \cite{McCurdy,bonchi14} can be used to model the generalised quantifiers of Barwise and Cooper \cite{BarwiseCooper81}. An instantiation of the abstract  setting to category of sets and relations  provided an  equivalent semantics to the  set theoretic semantics of generalised quantifiers. 

This paper stems from a  theoretical  question and a practical concern. On the theoretical side we have a  model for forming vectors for quantified sentences in distributional semantics \cite{HedgesSadr}. On the practical side, the vector space instantiation of this model relies on vector spaces being spanned by a power set object; this leads to an exponential increase in the size of the vector space and thus implementing it becomes computationally costly. 
On the other hand, there is the work of Zadeh in fuzzy quantifiers for natural language \cite{Zadeh1983}, which similar  to previous work provides a quantitative interpretation for the generalised quantifiers of Barwise and Cooper. Fuzzy sets have been applied to a variety of different domains, including to computational linguistics and information retrieval, for example see \cite{Novak,Cock,Zadeh:1996,Bezdek1978}. 
Given the mathematical equivalence  between fuzzy sets and vectors, the question arises whether there is a connection between the two settings of   vector representations of quantified sentences and their fuzzy set counterparts. In this paper we answer the question in positive and thus provide pathways to address the practical concern. As a result, we can now work with the fuzzy set counterparts  of compositional distributional vectors and avoid the pitfall of having to compute within an exponentially sized vector space.


The outline of the paper is as follows: we recall basic definitions of compact closed categories  and bialgebras over them and review how $\Rel$ and $\VRel$ are examples thereof. We then go through  the fact that the category $\VRel$  of sets and many valued relations models fuzzy sets and a logic over them. We  define in  $\VRel$  the  many valued versions  of the abstract quantifier interpretations  of  the setting of previous wok, where we worked with non-fuzzy sets and quantifiers \cite{HedgesSadr}. We show how Zadeh's fuzzy quantifiers  can be recast categorically in this setting and  prove that  Zadeh's  fuzzy  semantics of  quantified sentences is equivalent to  their corresponding  bialgebraic treatment. Whereas Zadeh's developments are not, at least explicitly, based on the  grammatical structure of  sentences, this result  indicates  that they do  inherently  follow the same composition principles as the ones  employed in compositional distributional semantics. We conclude our theoretical contributions by remarking on how  the degrees of truth obtained in the fuzzy interpretations relate to the absolute truth values of previous work. Overall,  we have taken  a step forward towards implementing and experimenting with quantifiers in distributional semantics, we,  however, leave experimenting with this model  to another paper.

This paper builds on the developments of a  previous technical report \cite{DostalSadr}.

\section{Dedication}
In 2001,  the second author of this article defended  her masters thesis under the supervision of Mohammed Ardeshir in Sharif University of Technology, Tehran, Iran. The content of the thesis was based on Brouwer's interpretation of Intuitionistic Logic and the motivation behind it came from  a mathematical logic course taught by Ardeshir on G\"odel's  incompleteness results, his work  on intuitionistic logic, and his translation of intiotionistic logic to modal logic $S4$.  The material taught in that course encouraged the second author to study the relationship between modal and constructive logics, work  on  a notion of constructive knowledge in epistemic logic \cite{MarionSadr2004,Sadr2003}, and    on dynamics of  belief update  in   logics based on monoids and lattices, including Heyting Algebras  \cite{SadrThesis,DyckhoffSaftTruf2013}. 

 The connection between this article and Areshir's work is through the relationship between intuitionistic logic and many valued logics.   G\"odel's observation that intuitionistic loigic cannot be characterised by finite truth tables \cite{Godel1932}, led to the axiomatisaiton suggested by Dummett in what is nowadays known as the G\"odel-Dummett logic \cite{Dummett1959}. 
Many valued  or fuzzy logics formalise the theory of fuzzy sets, put forwards by Zadeh \cite{Zadeh1965}.  Fuzzy sets assign a degree of membership to the elements of a set. This degree is usually a real number in the unit interval $[0,1]$. The degree of membership is used to define a truth-value between 0 and 1 for the formulae of a many valued logic. In these logics, 0 is still false and 1 is still true, a number between the two is a degree of truth.  

Algebraically speaking, a many valued logic is put together by a bounded lattice and a monoid. The G\"odel-Dummett logic has alternatively been seen as a many valued  logic where the monoid multiplication is idempotent and thus it coincides with the lattice operation of least upper bound.

\section{Generalised Quantifiers in Natural Language}
\label{sec:GenQuantNat}

We briefly review the theory of generalised quantifiers in natural language as presented in \cite{BarwiseCooper81}.  
Consider the fragment of English generated  by the following context free grammar: 
%
%
\vspace{-1em}
\begin{center}
\begin{tabular}{@{\hskip 0.5em}c@{\hskip 0.3em}c@{\hskip 0.4em}l}
S & $\to$ & NP VP \\
VP & $\to$ & V NP \\
NP &$\to$ & Det N\\
NP & $\to$ &  John, Mary, something, ...\\
N& $\to$ &   cat, dog, man, ...\\
VP & $\to$ &  sneeze, sleep, ...\\
V & $\to$ &   love, kiss, ...\\
Det & $\to$ &  some, all, no, most, almost all, several, ... 
\end{tabular}
\end{center}
A model for the language generated by this grammar  is a pair $(U, \semantics{\ })$, where $U$ is a universal reference set and $\semantics{\ }$ is an inductively defined interpretation function presented below. 
\begin{enumerate}
\item $\semantics{\ }$ on terminals:
\begin{enumerate}
\item The interpretation of a  determiner  $d$ generated by  `$\mbox{Det} \to d$'   is  a map  with the following type:
\[
 \semantics{d} \colon {\cal P}(U) \to {\cal P}{\cal P}(U)
\]
It  assigns to each $A \subseteq U$, a family of subsets of $U$. The images of these interpretations are referred to as \emph{generalised quantifiers}. For logical quantifiers, they are  defined as follows:
\begin{eqnarray*}
\semantics{\mbox{some}}(A) &=& \{X \subseteq U \mid X \cap A \neq \emptyset\}\\
\semantics{\mbox{every}}(A) &=& \{X \subseteq U \mid A \subseteq X\}\\
\semantics{\mbox{no}}(A) &=& \{X \subseteq U \mid  A \cap X = \emptyset\}\\
\semantics{n}(A) &=& \{X \subseteq U \mid \ \mid X \cap A \mid = n\}
\end{eqnarray*}

\noindent
A similar method is used to define non-logical quantifiers, for example ``most A'' is defined to be the set of subsets of U that has `most' elements of A, ``few A'' is the set of subsets of U that contain `few' elements of A, and similarly for `several' and `many'.  

\noindent
Generalising the two cases above, provides us with the following definition for any generalised quantifier $d$:
\begin{eqnarray*} \label{eq:genquant}
 \semantics{d}(A)  &=& \{X \subseteq U \mid  \quad X \ \mbox{has} \ d \ \mbox{elements of} \ A\}
\end{eqnarray*}

\item The interpretation of a  terminal $y \in \{np, n, vp\}$ generated by either of the rules `NP $\to$ np, N $\to$ n, VP $\to$ vp'  is $\semantics{y} \subseteq U$. That is, noun phrases, nouns and verb phrases are interpreted as subsets of  the reference set. 
\item The interpretation of a  terminal $y$ generated by the rule V $\to$ y is  $\semantics{y} \subseteq U \times U$. That is,  verbs are interpreted as binary relations over the reference set.  
\end{enumerate}
\item $\semantics{\ }$ on non-terminals:
\begin{enumerate}
\item The interpretation of  expressions generated by the  rule `$\mbox{NP} \to \mbox{Det N}$'   is  as follows:
\begin{eqnarray*}
&& \semantics{\mbox{Det N}} = \semantics{d}(\semantics{n}) \ \mbox{where} \ X \in \semantics{d}(\semantics{n}) \ \mbox{\bf iff}\\
&& X \cap \semantics{n} \in \semantics{d}(\semantics{n}) \ \mbox{for} \ \mbox{Det} \to d \ \mbox{and} \  \mbox{N} \to n 
\end{eqnarray*}


\item The interpretations of expressions generated by other rules are as usual:
\begin{eqnarray*}
\semantics{\mbox{V NP}} &=&    \semantics{v}(\semantics{np})  \\
\semantics{\mbox{NP VP}} &=& \semantics{vp}( \semantics{np})
\end{eqnarray*}
Here,  for  $R \subseteq U \times U$ and $A \subseteq U$, by $R(A)$ we mean the forward image of $R$ on $A$, that is $R(A) = \{y \mid (x,y) \in R, \mbox{for}  \ x \in A\}$. To keep the notation unified, for $R$ a unary relation $R \subseteq U$, we use the same notation and define $R(A) = \{y \mid y \in R, \mbox{for} \ x \in A \}$, i.e. $R \cap A$. 
\end{enumerate}
\end{enumerate}
The expressions generated by the rule `NP $\to$ Det N' satisfy a property referred to by \emph{living on} or \emph{conservativity}, defined below.

\begin{definition}\label{def:livingon}
For a terminal $d$ generated by the rule `Det $\to d$', we say that  $\semantics{d}(A)$  lives on $A$ whenever   $X \in \semantics{d}(A)$ iff  $X \cap A \in \semantics{d}(A)$, for $A, X \subseteq U$. Whenever this is the case, the quantifier $\semantics{d}$ is called a {\bf conservative} quantifier. 
\end{definition}
\noindent Barwise and Cooper argue that conservativity is a property of natural language quantifiers. This is certainly the case for the quantifiers generated  from the grammar given above. Thus for  the rest of the paper, we are assuming that our quantifiers are conservative. 
The `meaning' of a sentence  in this setting is its truth value. This is defined for any general sentence as follows:
\begin{definition}
\label{def:truth-genquant}
A sentence is true  iff $\semantics{\mbox{NP VP}} \neq \emptyset$ and false otherwise.
\end{definition}

\noindent
For the special cases of quantified subject and object phrases of interest to this paper,  a truth value is defined as follows: 

\begin{definition}
\label{def:truth-genquantMain}

\begin{enumerate}
\item
A sentence of the form  `Det N VP'  is \emph{true}  iff $\semantics{\mbox{Det N VP}} = \semantics{vp} \cap 
\semantics{n} \in \semantics{\mbox{Det N}}$ and \emph{false} otherwise. 
\item A sentence of the form `NP V Det N'  is \emph{true} iff  $\semantics{\mbox{NP V Det N}}  =   \semantics{n} \cap 
\semantics{v}(\semantics{np}) \in \semantics{\mbox{Det N}}$ and \emph{false} otherwise. \end{enumerate}
\end{definition}
As examples, first consdier sentence `some men sneeze'   with a quantifier at the subject phrase. This sentence is true iff $\semantics{\text{sneeze}} \cap \semantics{\text{men}} \in \semantics{\mbox{some men}}$, that is, whenever the set of things that sneeze and are men is a non-empty set.   Part of this meaning is obtained by following the inductive definition of $\semantics{\ }$ and part of it by applying Definition \ref{def:truth-genquantMain}.     The inductive case 2.b tells us that the semantics of this sentence is $\semantics{\mbox{NP \ VP}} = \semantics{\text{sneeze}}(\semantics{\mbox{some men}})$, where $\semantics{\text{NP}}$ is obtained by case 2.a and by unfolding it to $\semantics{\mbox{Det N}}$.  These unfoldings, when used in Definition \ref{def:truth-genquantMain}, provides us with the suggested  meaning above, that is \emph{true} iff $\semantics{{vp}} \cap \semantics{{n}} \in \semantics{\mbox{Det N}}$ and false otherwise. 
Similarly, as an example of a  sentence with a quantified phrase at its object position, consider   `John liked some trees' . This is  true iff $\semantics{\text{trees}} \cap \semantics{\text{like}}(\semantics{\text{John}}) \in \semantics{\mbox{some trees}}$, that is, whenever, the set of things that are liked by John and are trees is a non-empty set. Similarly, the sentence `John liked five trees' is true iff the set of things that are liked by John and are trees has five elements in it.

\section{Category Theoretic  Definitions}
\label{sec:PrelimCategoryTh}

This subsection briefly reviews compact closed categories and bialgebras. For a formal presentation, see \cite{KellyLaplaza80,Kock72,McCurdy}.  A compact closed category $\cC$ has objects $A, B$, morphisms $f \colon A
\to B$, and a monoidal tensor $A \otimes B$ that has a unit $I$; that is, we have $A \otimes I \cong I \otimes A \cong A$. Furthermore, for
each object $A$ there are two objects $A^r$ and $A^l$ and  the
following morphisms:
\begin{align*}
A \otimes A^r \stackrel{\epsilon_A^r} {\longrightarrow} \; 
I
\stackrel{\eta_A^r}{\longrightarrow} A^r \otimes A \\
A^l \otimes A \stackrel{\epsilon_A^l}{\longrightarrow} \; I
\stackrel{\eta_A^l}{\longrightarrow} A \otimes A^l\
\end{align*}
These morphisms satisfy the following equalities, where $1_A$ is the
identity morphism on object $A$:
\begin{align*}
 (1_A \otimes \epsilon_A^l) \circ (\eta_A^l \otimes 1_A)  = 1_A \\
(\epsilon_A^r \otimes 1_A) \circ (1_A \otimes  \eta_A^r)   = 1_A\\
 (\epsilon_A^l \otimes 1_A) \circ (1_{A^l} \otimes  \eta_A^l) = 1_{A^l} \\
(1_{A^r} \otimes \epsilon_A^r) \circ (\eta_A^r \otimes 1_{A^r}) = 1_{A^r}
\end{align*}
\noindent These express the fact that $A^l$ and $A^r$ are the left and right
adjoints, respectively, of $A$. A self adjoint compact closed category is one in which for every  object $A$ we have $A^l \cong A \cong A^r$.

Given two compact closed categories ${\cal C}$ and ${\cal D}$,  a strongly monoidal functor $F \colon {\cal C} \to {\cal D}$  is defined as follows:
\[
F(A \otimes B) = F(A) \otimes F(B) \qquad
F(I) = I
\]
One can show that this functor preserves the compact closed structure, that is we have:
\[
F(A^l) = F(A)^l \qquad
F(A^r) = F(A)^r
\]
A bialgebra  in a  symmetric monoidal  category $({\cal
  C}, \otimes, I, \sigma)$ is a tuple $(X,  \delta, \iota, \mu, \zeta)$ where,
for $X$ an object of ${\cal C}$, the triple $(X, \delta, \iota)$ is  an internal comonoid; 
i.e.~the following are  coassociative and counital  morphisms of ${\cal
  C}$:
\begin{align*}
\delta \colon X \to X \otimes X&\qquad& \iota \colon X \to I
\end{align*}
Moreover $(X, \mu, \zeta)$ is  an internal  monoid; i.e.~the following are  associative and unital  morphisms:
\begin{align*}
\mu \colon  X \otimes X \to X  &\qquad& \zeta \colon I \to X
\end{align*}
Morphisms   $\delta$ and $\mu$ satisfy  four equations \cite{McCurdy}:
\begin{align*}
\iota \circ \mu &= \iota \otimes \iota		 \\
\delta \circ \zeta &= \zeta \otimes \zeta	\\
\delta \circ \mu &= (\mu \otimes \mu) \circ (\operatorname{id}_X \otimes \sigma_{X,X} \otimes \operatorname{id}_X) \circ (\delta \otimes \delta)  \\
\iota \circ \zeta &= \operatorname{id}_I
\end{align*}
Informally, the  co-multiplication $\delta$ copies  the information contained in
one object into two objects, and the  multiplication $\mu$ merges  the
information of two objects into one.

\medskip
\noindent
{\bf Example.  Sets and Relations}.
An example of a  compact closed category is
$\Rel$, the category of sets and relations. Here, $\otimes$ is
cartesian product with the singleton set as its unit $I = \{\star\}$, and  $A^l = A = A^r$.  Hence $\Rel$ is  self adjoint.  Given a set $S$ with elements $s_i, s_j \in S$,  the epsilon and eta maps are given as follows:
\begin{align*}
\epsilon^l  =  \epsilon^r \colon S \times S \relto   \{\star\} \quad  \mbox{given by} \\
(s_i, s_j) \epsilon \star  \iff &s_i = s_j \\[0.5em]
\eta^l = \eta^r \colon    \{\star\}  \relto S \times S \quad \mbox{given by} \ \\
\star \eta (s_i, s_j)  \iff  & s_i = s_j
\end{align*}

\noindent
These relations hold iff the first and second elements of the pair they are acting on are the same, i.e. whenever $s_i = s_j$, we related $\star$ to the pair $(s_i, s_j)$ via $\eta$ maps and relate the pair to the $\star$ via an epsilon map. They are designed to sift out pairs that are reflexive. 



For an object in $\Rel$ of the form $W = \mathcal P (U)$, \cite{HedgesSadr} gave $W$ a bialgebra structure by taking
\begin{align*}
\delta \colon   S \relto S \times S \quad \mbox{given by}\\
A \delta (B, C) \iff &A = B = C \\[0.5em]
\iota \colon S \relto    \{\star\}   \ \qquad \mbox{given by}\\
 A \iota \star \iff  & \text{ (always true)} \\[0.5em]
\mu  \colon   S \times S \relto S \quad \mbox{given by} \\ 
(A, B) \mu C \iff & A \cap B = C \\[0.5em]
\zeta  \colon  \{\star\}  \relto S \ \qquad \mbox{given by}\\
 \star \zeta A \iff & A = U
\end{align*}
It was shown in \cite{HedgesSadr}  that the  four axioms of a bialgebra hold for the above definitions. In order to obtain an intuition, the $\delta$ map relates a subset $A$ of the universe $U$ to a pair of subsets $(B,C)$ iff these three subsets are the same, i.e. iff $A = B = C$. This relation only holds when the set  in its first input is the same as the pairs of sets in its second input, where as $\iota$ is meant to be the relation that always holds.  The relation $\mu$ is more sophisticated, it is meant to enable the formalism to perform the set intersection operation, so given a pair of subsets of universe $A$ and $B$, and another subset thereof $C$, the relation $\mu$ holds iff $C$ is the intersection of $A$ and $B$. The unit of this map, i.e. $\zeta$ only holds when its input subset $A$ is actually the whole universe, the unit of intersection $A \cap U = A$. 

In the next section we show how the category of sets and many valued relations is also an example of a self adjoint compact closed category with bialgebras over it.  In fact both  $\Rel$ and  category of sets and many valued relations  are also dagger compact closed and have  other desirable properties, e.g. being  partial order enriched, for elaborations on these properties see \cite{Marsden2017}.

%
%
%


\section{Category of Sets and Many Valued Relations}
\label{sec:MVRelCat}

\begin{definition}[Commutative quantale] 
A commutative quantale $\V$ is a complete 
lattice $(V,\bigwedge,\bigvee)$ with the structure of 
a commutative monoid $(V, \tensor, \e)$ such that the 
tensor is monotone and distributes over arbitrary 
joins.

More in detail, $\V$ being a complete lattice means 
that $V$ is partially ordered by $\leq$ and every 
subset $V' \subseteq V$ has an infimum (or meet) and a 
supremum (or join), denoted by $\bigwedge V'$ and 
$\bigvee V'$ respectively. From this it follows that 
$V$ contains the greatest element $\top$ and the 
lowest element $\bot$. The fact that $(V, \tensor, 
\e)$ is a commutative monoid means that $\tensor$ is 
commutative, associative, and $\e$ is the identity 
element:
$$v \tensor \e=  v = \e \tensor v$$
The monotonicity of the tensor requires that $v \tensor w \leq v' \tensor w$ holds for $v \leq v'$, and distributivity of tensor 
over arbitrary joins means that the following equality is satisfied:
$$\left(\bigvee_i x_i \right) \tensor y = 
\bigvee_i(x_i \tensor y)$$
\end{definition}

\begin{definition}[Complete Heyting algebra]
A complete Heyting algebra $\V$ is a 
commutative quantale where $\tensor = \wedge$ and $\e 
= \top$. In other words, it is a complete lattice 
$(V,\bigwedge,\bigvee)$ where the meet operation 
distributes over arbitrary joins:
$$\left(\bigvee_i x_i \right) \wedge y = \bigvee_i(x_i 
\wedge y)$$
\end{definition}

\begin{definition}[G\"odel chain]\label{def:godelchain}
We say that a complete Heyting algebra $\V$ is a 
\emph{G\"odel chain} if the ordering relation $\leq$ 
of the underlying lattice of $\V$ is a linear order, 
that is, for two elements $v \neq v'$ it either holds 
that $v \leq v'$ or $v' \leq v$.
\end{definition}

\medskip
\noindent
{\bf Example.} 
Instances of commutative quantales:
\begin{enumerate}
\item The real interval $[0,1]$ with the usual lattice 
structure (given by computing suprema and infima), the 
tensor being the meet and the unit being 1, is a 
complete Heyting algebra, moreover a G\"odel chain.
\item The real interval $[0,1]$ with the usual 
structure, the unit 1 and the tensor being defined as
$$
a \tensor b = \max(0,a+b-1)
$$
is a commutative quantale.
\item The real interval $[0,1]$ with the usual 
structure, the unit 1 and the tensor being defined as
$$
a \tensor b = a \cdot b
$$
(multiplication) is a commutative quantale.
\item As a very special case, the 2-element Boolean 
algebra is a commutative quantale.
\end{enumerate}

\begin{definition}[Many-valued relation]
For a given quantale $\V$, a \emph{many-valued 
relation} $R: A \relto B$ is a function $R: A \times B 
\to V$. We view this function as a $\V$-valued matrix. We compose two relations $R: A \relto B$ and $S: B 
\relto C$ to get a relation $S \circ R: A \relto C$ 
such that
$$(S \circ R)(a,c) = \bigvee_{b \in B} (R(a,b) \tensor 
S(b,c))$$
holds in $\V$.
\end{definition}

\begin{definition}[The category of $\V$-relations]
The collection of all sets and of $\V$-relations 
between sets is a category. There is an identity 
$\V$-relation $\id_A$ for every set $A$:
$$\id_A (a,a') = 
\begin{cases}
\e & \text{ if } a = a' \\
\bot & \text{ otherwise}
\end{cases}
$$
An easy computation yields that $\V$-relation 
composition is associative. We denote the category of 
all sets and $\V$-relations as $\VRel$.
\end{definition}

\begin{remark}
The associativity of $\V$-relation composition follows 
from complete distributivity of $\V$. For 
$\V$-relations over finite sets, only finite 
distributivity of tensor over joins would be needed.
\end{remark}

\medskip
\noindent
{\bf Example.}
Some examples of $\VRel$ for various choices of $\V$:
\begin{enumerate}
\item When $\V$ is the 2-element Boolean algebra, 
$\VRel$ is the category $\Rel$ of sets and (ordinary) 
relations.
\item When $\V$ is the real interval $[0,1]$ with 
G\"odel operations $\min$ and $\max$, the category 
$\VRel$ has sets as objects, and the composition of 
morphisms ($\V$-relations) acts as follows. Given two 
$\V$-relations $R: A \relto B$ and $S: B \relto C$ (so 
two functions $R: A \times B \to [0,1]$ and $S: B 
\times C \to [0,1]$), the composite $S \circ R: A 
\relto C$ is given by
\[
(S \circ R)(a,c) = \max_{b \in B} \min(R(a,b),S(b,c)).
\]
Given yet another $\V$-relation $T: C \relto D$, the 
composite $T \circ S \circ R$ is then computed as 
follows:
\begin{align*}
(T \circ S \circ R)(a,d) = &\\
\max_{b \in B, c \in 
C} \min & (R(a,b),S(b,c),T(c,d)).
\end{align*}
\end{enumerate}

\begin{remark}
Observe that there is an inclusion functor
\[
\widetilde{(-)}:\Rel \to \VRel
\]
for any $\V$ with more than one element. Indeed, let 
the functor act as an identity on objects, and assign 
to a relation $R: A \relto B$ the $\V$-valued relation 
$\widetilde{R}: A \relto B$ defined as follows:
\[
\widetilde{R}(a,b) =
\begin{cases}
\e & \text{ if } R(a,b) \text { holds,} \\
\bot & \text{ otherwise.}
\end{cases}
\]
An easy computation yields that $\widetilde{\id_A} = 
\id_A$ and that $\widetilde{S \circ R} = \widetilde{S} 
\circ \widetilde{R}$.
\end{remark}

\begin{lemma}
The category $\VRel$ is a self adjoint compact closed 
category with the tensor being the cartesian 
product $\times$ and the unit $I$ 
being the singleton set $\{ 
\star \}$.
\end{lemma}
\begin{proof}
Let us define the epsilon maps $\epsilon_S: S \times S 
\relto I$ for each $S$ as follows
\[
\epsilon_S((a,b),\star) =
\begin{cases}
\e & \text{ if } a = b \\
\bot & \text{ otherwise}
\end{cases}
\]
and define the eta maps $\eta_S: I \relto S \times S$ 
similarly:
\[
\eta_S(\star,(a,b)) =
\begin{cases}
\e & \text{ if } a = b \\
\bot & \text{ otherwise}
\end{cases}
\]
Since with these definitions the epsilon and eta maps 
are the images of the epsilon and eta maps from $\Rel$ 
under the inclusion functor $\widetilde{(-)}: \Rel \to 
\VRel$, the axioms of a compact closed category hold 
in $\VRel$. It remains to show that $\epsilon$ and 
$\eta$ are natural; but this is straightforward.
\end{proof}

\begin{remark}
Let us fix a set $U$. Very similarly to the case of 
$\Rel$, we can define a bialgebra over the set $S = 
P(U)$ in $\VRel$ by the following data. The relation 
$\delta: S \relto S \times S$ is defined as
\[
\delta(A,(B,C)) =
\begin{cases}
\e & \text{ if } A = B = C \\
\bot & \text{ otherwise}.
\end{cases}
\]
The relation $\mu: S \times S \relto S$ is defined as
\[
\mu((A,B),C) =
\begin{cases}
\e & \text{ if } A \cap B = C \\
\bot & \text{ otherwise}.
\end{cases}
\]
The relation $\iota: S \relto I$ is defined as
\[
\iota(A,\star) = \e \text{ for every } A.
\]
The relation $\zeta: I \relto S$ is defined as
\[
\zeta(\star,A) =
\begin{cases}
\e & \text{ if } A = U \\
\bot \text{ otherwise.}
\end{cases}
\]
In fact, we obtain the structure of a bialgebra over 
$P(U)$ in $\VRel$ by taking the bialgebra structure 
over $P(U)$ in $\Rel$ and applying the inclusion 
functor $\widehat{(-)}$.
\end{remark}

\section{Fuzzy Sets and Fuzzy Quantifiers}
\label{sec:ZFuzzy}

In this section we review  definitions of fuzzy sets and quantifiers, as done by Zadeh \cite{Zadeh1983}. 
A  fuzzy set is a set whose  elements have a corresponding weight  associated to them. For a set $A$, the weight $\mu_i$ of  element  $u_i$ is  interpreted  as  the degree of membership of $u_i$ in $A$. The fuzzy set $A$ is represented symbolically by the following sum:
\[
A = \mu_1 u_1 + \mu_2 u_2 +  \cdots + \mu_n u_n
\]
standing for the following set of pairs of weights and elements:
\[
\{(\mu_1, u_1), (\mu_2, u_2), \cdots, (\mu_n, u_n)\}
\]
The sum above  denotes  a union operation on  sets  containing   single $\mu_i u_i$  elements, where $\mu_i u_i$  stands for  the pair $(\mu_i, u_i)$. Non-fuzzy, aka 
\emph{crisp}, sets are  special instances of fuzzy ones, where for every $u_i$ of the set, we have  $\mu_i = 1$, in other words:
\[
A = u_1 + u_2 + \cdots + u_n
\]


The   \emph{absolute}  cardinality of a  fuzzy set is defined via the notion of  \emph{sigma-count},  defined below:
\[
\Sigma Count(A) = \Sigma_{i=1}^n \mu_i
\]
This is the arithmetic sum of the degrees of membership in $A$; it is, if needed,  rounded to the nearest integer. Terms whose degrees of membership fall below a certain threshold, may be omitted from the sum. This is to avoid a situation where a large number of terms with low degrees become equivalent to a small number of terms with high degrees.   
Following Zadeh, we  denote the  \emph{absolute}  cardinality of a non-fuzzy set by $Count(A)$. Observe that, under the interpretation of a non-fuzzy set as a special instance of a fuzzy set, we have that $Count(A) = \Sigma Count(A)$. 

The \emph{relative} cardinality of a  fuzzy set is a possibility distribution over the cardinality  of that set, denoted as follows
\[
  \Pi_{\Sigma Count(A)}
\]

The  quantified sentences  Zadeh considers are built from two basic  forms: ``There are $Q$ $A$'s" and ``$Q$ $A$'s are $B$'s".  Each of these propositions induces a possibility distribution. Zadeh provides the following insights for the analysis of these quantified propositions. ``There are $Q$ $A$'s" implies that  the probability of event $A$ is a fuzzy probability equal to $Q$. ``$Q$ $A$'s are $B$'s" implies that the conditional probability of event $B$ given event $A$  is a fuzzy probability which is equal to $Q$. Most statements involving fuzzy probabilities may be replaced by semantically equivalent propositions involving fuzzy quantifiers.  The connection between this two, Zadeh reports, plays an important role in expert systems and fuzzy temporal logic and has been developed in previous work of Zadeh  \cite{BarrFeigenbaum1983}.  

According to Zadeh, fuzzy quantifiers should be treated as fuzzy numbers. A fuzzy number provides a fuzzy characteristic of the absolute or relative cardinality of one or more fuzzy or non-fuzzy sets. As an example, consider the fuzzy quantifier ``most" in the proposition  ``Most big men are kind".  This proposition is interpreted as a fuzzily defined proportion of the fuzzy set  ``kind men" in the fuzzy set  ``big men". If our sentence was ``Vickie has several credit cards",  then ``several" would be  a fuzzy characterisation of the cardinality of the non-fuzzy set  ``Vickie’s credit cards".   The notion of the cardinality of a fuzzy set helps us compute the proposition ``Vickie has several credit cards". Here, ``most" is a fuzzy characterisation of the relative cardinality of the fuzzy set  ``kind men" in the fuzzy set  ``big men".  It might not always be clear  how a constituent  fuzzy number  relates to a fuzzy quantifier, but we will not go in details of these here, for examples see \cite{Zadeh1983}.


The fuzzy semantics of a  proposition   $p$ is interpreted as ``the degree of truth of $p$", or the possibility of $p$. In order to compute this, we  translate $p$  into a \emph{possibility assignment equation}, which is denoted  as follows
\[
\Pi_{(X_1, \cdots, X_n)} = F
\]
where $F$ is a fuzzy subset of the universe of discourse $U$ and $\Pi_{(X_1, \cdots, X_n)} $ is the joint possibility distribution over  (explicit or implicit) variables $X_1, \cdots, X_n$ of $p$.   For instance, the proposition ``Vickie is tall" is translated as follows:
\[
\Pi_{Height(Vickie)} = TALL
\]
Here, $TALL$ is a fuzzy subset of the real line, $Height(Vickie)$ is a variable implicit in ``Vickie is tall",  and $\Pi_{Height(Vickie)} $ is the possibility distribution of this variable.  The above possibility assignment equation implies that 
\[
Poss\{ Height(Vickie) = u\} = \mu_{TALL} (u)
\]
where $Poss\{X = u\}$  the possibility that $X$ is $u$, for $u$ a specified value. The above thus states that  ``the possibility that height of Vickie is $u$ is equal to  $\mu_{TALL} (u)$, that is,  is the grade of membership of $u$ in the fuzzy set $TALL$.  Quantified sentences are translated in a similar way. For instance,  ``Vickie has several credit cards", is translated to the following:
\[
\Pi_{Count(Credit-Cards(Vickie))} = SEVERAL
\]
Suppose that 4 is compatible with the meaning of ``several" with degree $0.8$, then the above implies that, for instance, the possibility that Vickie has 4 credit cards is
\[
Poss\{Count(Credit\mbox{-}Cards(Vickie)) = 4\} =0.8
\]
In order to analyse sentences of the general forms ``There are $Q$ $A$'s"  and `$Q$ $A$'s are $B$'s", Zadeh assumes that they are semantically equivalent to the following:
\begin{eqnarray*}
\mbox{There are} \ Q   \ A\mbox{'s} &\leadsto& \Sigma Count(A) \ \mbox{is}  \  Q\\
  Q \ A\mbox{'s are} \ B\mbox{'s} &\leadsto& Poportion(B|A) \ \mbox{is} \  Q
 \end{eqnarray*}
Here,  $Poportion(B|A)$  is the proportion of elements of $B$ that are in $A$, aka the relative cardinality of $B$ in $A$, formally defined as follows:
\[
\Pi_{Poportion(B|A)} := \frac{\Sigma Count (A \cap B)}{\Sigma Count(A)}
\]
Both $Proportion(B|A)$ and $\Sigma Count(A)$ may be fuzzy or non-fuzzy counts. Zadeh then  formalises the above counts as possibility assignment equations as follows
\begin{eqnarray*}
\Sigma Count(A) \ \mbox{is}  \  Q &\leadsto& \Pi_{\Sigma Count(A)} = Q\\
 Proportion(B|A) \ \mbox{is} \  Q &\leadsto& \Pi_{Proportion(B|A)} = Q
\end{eqnarray*}

In the spirit of truth-conditional semantics, the weight of each  of the elements  of the set can be  interpreted as the degree of truth of the proposition denoted by the element. This weight is    $Q(\Sigma Count(A))$ for sentences of the form ``There are $Q$ $A$'s" and $Q(Proportion(B|A))$ for sentences of the form ``$Q$ $A$'s are $B$'s".

%
%

Writing $\mu_A(u)$ for the degree of membership of $u$ in the fuzzy set $A$, we define the intersection of two fuzzy sets $A$ and $B$ as
\[
A \cap B = \Sigma_{i} \ \min(\mu_A(u_i),\mu_B(u_i)) \ u_i
\]
where $i$ is understood to range over all the elements in $A$ and $B$ (when an element is in $A$ but not in $B$ it will still be represented in $A$ with a degree of $0$). A similar version without the $\Sigma$ is used to define it for the non-fuzzy case.

\medskip
\noindent
{\bf Example.}
Let's say we have a universe
\[
U = \{ u_1, u_2, u_3, u_4, u_5 \}
\]
and fuzzy sets for ``kind people" and ``big men" as follows:
\begin{eqnarray*}
KP& =&  0.5 u_1 + 0.8 u_2 + 0.2 u_3 + 0.6 u_4 \\
BM& =& 0.8 u_1 + 0.3 u_2 + 0.1 u_3 + 0.9 u_4 + 1 u_5
\end{eqnarray*}
The quantified sentence  ``Most big men are kind", is translated to the following possibility assignment equation: 
\[
\Pi_{Proportion(KP|BM)} = MOST
\]
The intersection of $KP$ and $BM$ is computed as follows:
\[
KP \cap BM = 0.5 u_1 + 0.3 u_2 + 0.1 u_3 + 0.6 u_4
\]
The proportion of big men that are kind is computed as follows:
\begin{eqnarray*}
Proportion(KP | BM) = \frac{\Sigma Count(BM \cap KP)}{\Sigma Count(BM)} =  \frac{0.5 + 0.3 + 0.1 + 0.6}{0.8 + 0.3 + 0.1 + 0.9 + 1 }\ = \frac{1.5}{3.1}\\
\end{eqnarray*}
Suppose that proportions between  0.6 and 0.7 are compatible with the meaning of $MOST$ with degree 0.75.  Then, since $ \frac{1.5}{3.1} = 0.48$, the degree of truth of our sentence is below  0.75. For the crisp quantifier $ALL$, the sentence ``All big men are kind" is, since only  the proportion 1 is compatible with the meaning of $ALL$ with degree 1, which is not the case here. 

Possibility distributions are  encodable into vectors and indeed the possibility distributions of fuzzy quantifiers are learnt by Zadeh via a  test-score procedure where as vectors by sampling from a database of related data.

\section{Fuzzy Quantified Sentences  in $\VRel$}
\label{sec:FuzzyMVRel}

A non-fuzzy generalised quantifier $d$ is interpreted as a relation  $\llbracket d \rrbracket$ over the power set of the universe of discourse $P(U)$, where it  relates a subset $A \subseteq U$ to subsets  $u_i \subseteq U$, based on the cardinalities of $A$ and $u_i$, as defined in section \ref{sec:GenQuantNat}.  The fuzzy version of this quantifier is interpreted as a many valued  relation over $P(U)$, where, in fuzzy set notation, it relates $A$ to subsets $u_i \subseteq U$ and assigns to each such subset a degree of membership $\mu_i$. The result is a fuzzy set whose weights come from a  possibility distribution over the  relative cardinalities  of $A$ and $u_i$'s.  In Zadeh's notation: 
\begin{equation}
\label{eq:fuzzygenquant}
{\llbracket d \rrbracket}(Proportion (u_i|A))  = \mu_i
\end{equation}

%
%
%

\noindent
We translate the above  in the language of $\VRel$, referring to the categorical version of the  fuzzy generalised quantifier by $\overline{\llbracket d \rrbracket}$, which is a map with the type $P(U) \relto P(U)$. This is the map which was denoted by $Q$ in the fuzzy generalised quantifier setting. In order to be coherent with the categorical semantics of sets and relations, we use the notation $\overline{\llbracket d \rrbracket}$ for it. Recall that in the categorical generalised quantifier theory, a quantifier was represented by $\llbracket d \rrbracket$. In what follows, we first define  $\overline{\llbracket d \rrbracket}$ as a generalised fuzzy quantifier in the categorical setting of many valued relations,  this is in Definition \ref{def:mvrelFuzzy}. Then, in Definition \ref{def:fuzzymodel}, extend it to quantified sentences of the fragment of language generated by the preliminary grammar of Section 3. 

\begin{definition}
\label{def:mvrelFuzzy}
For $\V = [0,1]$ and given a fuzzy generalised  quantifier  for which we have assumed $\Pi_{Proportion(B|A)} = \overline{\llbracket d \rrbracket}$, we define its $\VRel$ encoding to be  the many valued relation $\overline{\llbracket d \rrbracket}: P(U) \relto P(U)$, with values coming from the possibility distribution of $\overline{\llbracket d \rrbracket}$, defined  as follows:
\[
\overline{\llbracket d \rrbracket}(A,B) = \mu_i \quad \mbox{for} \quad \mu_i =  \semantics{d}(Proportion(B|A)) 
\]
\end{definition}

\noindent 
In order to obtain a  many valued relation in $\VRel$, we need a numerical value assigned  to subsets $A$ and $B$ of universe. This number is nothing but  the weight of $ \semantics{d}(Proportion(B|A))$, denoted by $\mu_i$ in equation \ref{eq:fuzzygenquant}. If for any reason this number is unattainable, e.g. when $B$ and $A$ are not related to each other at all, we assign the $\bot$ to it. 

\begin{remark} 
\label{remark:cons}
Conservativity of a quantifier $d$ in $\VRel$ is defined as follows:
\[
 \overline{\semantics{d}}(A,B) =  \overline{\semantics{d}}(A, A \cap B) 
\]
and is implied  by its conservativity in $\Rel$. This is because 
\begin{align*}
Proportion(A \cap B|A) =&  \frac{\Sigma Count (A \cap B \cap A)}{\Sigma Count (A)}  \\
= \frac{\Sigma Count (A \cap B)}{\Sigma Count (B)} &= Proportion(B|A)
\end{align*}
\end{remark}
The absolute quantifiers such as  ``every''  and ``some" can still be  interpreted in this setting, by defining them as follows:
\begin{align*}
\overline{\llbracket \text{every} \rrbracket}(A,B) =&
\begin{cases}
\e & \text{ if } A \subseteq B, \\
\bot & \text{ otherwise.}
\end{cases}\\[0.2em]
\overline{\llbracket \text{some} \rrbracket}(A,B) =&
\begin{cases}
\e & \text{ if } A \cap B \neq \emptyset, \\
\bot & \text{ otherwise.}
\end{cases}
\end{align*}

\noindent
The fuzzy version of a model  generated by the grammar of  Section \ref{sec:GenQuantNat},  becomes as follows:

\begin{definition}
A fuzzy model $(U,\llbracket \; 
\rrbracket)_f$ is one  where  for  $A \subseteq U$ we have:
\[
\llbracket A \rrbracket :=\mu_1 u_1 + \mu_2 u_2 + \cdots + \mu_n u_n\quad \mbox{for} \ u_i \in U
\]
\end{definition}
The fuzzy semantics of expressions of the grammar are  many valued versions of those of  $(U,\llbracket \; \rrbracket)$.  For reasons of space, we only give them in $\VRel$ notation. 

The following definition explains the many valued semantics of a sentence in our example grammar is computed. 

\begin{definition} \label{def:fuzzymodel}
A $\VRel$   fuzzy model is the  tuple
$(\VRel,P(U),I,\overline{\llbracket \; 
\rrbracket})$ for which we have the following interpretation:

\begin{enumerate}
\item A terminal $x$ of either category N,NP, or VP is interpreted as  a many valued relation whose value is the degree to which a subset $A$ of the universe is  $\semantics{x}$. This is  the relative sigma count of the subset  $A$ in $\semantics{x}$, that is: 

\[
\star  \overline{\llbracket x \rrbracket}  A  := Proportion(A|\semantics{x})
\]



\item A terminal $x$ of category V is interpreted as a many valued relation whose value is the degree to which its image on a subset $A$ of universe is a subset $B$ of the universe, that is the relative sigma count of $B$ in $\semantics{x}(A)$ :
\[
\star \overline{\semantics{x}} (A,B) = Proportion (B|\semantics{x}(A))
\]
where $\semantics{x}(A)$ is the  application of $\semantics{x}$ to $A$, resulting in a set $\Sigma_{i=1}^n  \mu_i b$ where the subscripts of the  $\mu$'s vary over elements of fuzzy sets $A$ and $\semantics{v}$, so we have 
\[
 \max_{a_i} \min(\mu_A(a_i),\mu_{\semantics{v}}(a_i,b_i))
 \]
 Here,  $\mu_A$ and  $\mu_{\semantics{v}}$  are degrees of memberships of elements of fuzzy sets $A$ and $\semantics{v}$, respectively. 
\end{enumerate}
\end{definition}
Given this definition, we compute the
many valued semantics of quantified sentences and show that they are equivalent to the fuzzy quantifier definitions of Zadeh. Note that for $U$ the universe of reference, the relative cardinality of $A$ in $U$ is the same as  the cardinality of A. Thus "There are Q A's" has the same fuzzy meaning as "Q U's are A's".

\begin{proposition}
\label{prop:qsvo}
The many valued semantics of a  sentence with a quantified subject ``\it{d  np vp}''  is the same as  its  fuzzy quantifier semantics in $\VRel$. 
 \end{proposition}
\begin{proof}
The many valued semantics of   ``{\it d np vp}'' is computed in four steps, according to the four composed morphisms  of its  $\VRel$ semantics: 
\[
\epsilon \circ (\overline{\semantics{d}} \otimes \mu) \circ (\delta \otimes id) \circ (\overline{\semantics{np}} \otimes \overline{\semantics{vp}})
\]
In the first step,  we compute the following map
\[
\overline{\semantics{np}} \otimes \overline{\semantics{vp}}  \colon \{\star\} \otimes \{\star\} \relto {\cal P}(U) \otimes  {\cal P}(U)
\]
For   $A, B \subseteq U$,   the value returned by this map is
\[
(\star, \star) (\overline{\semantics{np}} \otimes \overline{\semantics{vp}})(A, B) = \min \left(\star\overline{\semantics{np}} A, \star \overline{\semantics{vp}} B \right)
\]
In the second step, we compute the following  map
\[
 (\delta \otimes id)  \circ (\overline{\semantics{np}} \otimes \overline{\semantics{vp}})
 \colon \{\star\} \otimes \{\star\}  \relto {\cal P}(U) \otimes  {\cal P}(U) \otimes  {\cal P}(U)
\]
For   $C, D, E \subseteq U$,  it  returns the following value
\[
(\star, \star) 
 (\delta \otimes id)  \circ (\overline{\semantics{np}} \otimes \overline{\semantics{vp}})
 ((C, D), E)
\]
which is equal to 
\[
\max_{(A, B)}  \min  \left( (\star, \star) (\overline{\semantics{np}} \otimes \overline{\semantics{vp}}) (A, B), \right .  
\left . (A, B)  (\delta \otimes id) ((C, D), E)  \right .  \Big)
\]
The maximum value of the above term is realised for $(A, B)$'s for which we have  $A = C = D$ and  $B = E$,  in which case this value becomes equal to the following 
\[
\min
\left
(\star \overline{\semantics{np}} A, \star \overline{\semantics{vp}} B
\right)
\]
This is since the $\delta$ and $id$ maps return $e$ in their best case and $e$ is the unit of the $\min$ operation. 

\noindent 
 In the third step, we compute the following  map:
\[
(\overline{\semantics{d}} \otimes \mu) \circ  (\delta \otimes id) \circ (\overline{\semantics{np}} \otimes \overline{\semantics{vp}})
 \colon
\{\star\} \otimes \{\star\}  \relto {\cal P}(U) \otimes {\cal P}(U) 
\]
The value generated by this map is
\[
(\star, \star) (\overline{\semantics{d}} \otimes \mu) \circ (\delta \otimes id) \circ (\overline{\semantics{np}} \otimes \overline{\semantics{vp}})
(F, G)
\] 
equal to 
\[
 \max_{((C, D), E)} \min 
\Big  (
 (\star, \star) (\delta \otimes id) \circ
  (\overline{\semantics{np}} \otimes \overline{\semantics{vp}}) ((C, D), E),
 (C, (D, E))  (\overline{\semantics{d}} \otimes \mu)  (F,  G)
 \Big )
\]
 The maximum of the above  is realised for the $C, D, E$ that make $ G = D \cap E$  true,  in which case this value becomes equal to
 \[
 \max_{(A,B)}
 \min
\left
(\star \overline{\semantics{np}} A, \star \overline{\semantics{vp}} B
, A \overline{\semantics{d}} F
\right)
 \]
 In the fourth step we compute the full map
\[
 \epsilon \circ (\overline{\semantics{d}} \otimes \mu)  \circ (\delta \otimes id) \circ (\overline{\semantics{np}} \otimes \overline{\semantics{vp}}) \colon\\
 \{\star\} \otimes \{\star\} \relto \{\star\}
\]
 the value generated by which is 
 \[
 (\star, \star)\,  \epsilon \circ (\overline{\semantics{d}} \otimes \mu) \circ (\delta \otimes id) \circ (\overline{\semantics{np}} \otimes \overline{\semantics{vp}}) \star
 \]
which is  equal to 
\[
 \max_{(F,G)} \min
 \Big(
 (\star, \star) (\overline{\semantics{d}} \otimes \mu) \circ (\delta \otimes id) \circ 
 (\overline{\semantics{np}} \otimes \overline{\semantics{vp}}) (F, G),
 (F, G) \epsilon \star
 \Big )
\]
 The maximum of this term is realised when we have $F = G$, in which case it becomes equal to 
 \[
 \max_{(A,B)}
 \min
\left
(\star \overline{\semantics{np}} A, \star \overline{\semantics{vp}} B
, A \overline{\semantics{d}} A \cap B				
\right)
 \]
By applying Definition \ref{def:fuzzymodel}, the above unfolds as follows 
\[
 \max_{(A,B)}
 \min\Big(
Proportion (\semantics{np}, A),Proportion (\semantics{vp},B), 
\semantics{d}[Proportion(A \cap B | A)]
\Big)
\]
where we are assuming  $\Pi_{Proportion(A \cap B |A)} = d$. 
Since  our quantifiers are conservative, we  apply the simplifications computed in Remark \ref{remark:cons}, then  given that the first two terms of the above $\max\min$  are maximised when  $A = \semantics{np}$, $B = \semantics{vp}$, the above simplifies to the following
\[
 \max_{(A,B)}
 \min
\Big
(
\frac{\Sigma Count(\semantics{np} \cap \semantics{np})}{\Sigma Count(\semantics{np})}, \frac{\Sigma Count(\semantics{vp} \cap \semantics{vp})}{\Sigma Count(\semantics{vp})}, \semantics{d}\Big[\frac{\Sigma Count(A \cap B)}{\Sigma Count(A)}\Big]
\Big
)
\]
which simplifies to


 \[
 \semantics{d}\Big[\frac{\Sigma Count(\semantics{np} \cap \semantics{vp})}{\Sigma Count(\semantics{np})}\Big] \] 
Again, here we are assuming that   $\Pi_{Proportion(\semantics{vp} | \semantics{np})} = d$.  Observe now that this  equals the following 
 \[ 
 \semantics{d}[Proportion(\semantics{vp}|\semantics{np})] \quad \ \text{for} \ \Pi_{Proportion(\semantics{vp} | \semantics{np})} = d
 \] 
which is the same as  Zadeh's fuzzy quantifier semantics of ``d np's are vp's".
 \end{proof}

\medskip
\noindent
{\bf Example.}
Given Definition 10, the statement  ``several cats sleep" will be interpreted as
\[
 \max_{(A,B)} \min
\Big
(\star \overline{\semantics{cats}} A, \star \overline{\semantics{sleep}} B, 
A \overline{\semantics{several}} A \cap B			
\Big)
\]

\noindent This will be maximised for $A = \semantics{cats}, B = \semantics{sleep}$ and  when assuming that  $\Pi_{Proportion(A \cap B | A)} = \text{several}$, in which case the value of the statement will become

$$ \semantics{several}\Big[\frac{\Sigma Count(\semantics{cats} \cap \semantics{sleep})}{\Sigma Count(\semantics{cats})}\Big] $$
To compute this concretely, suppose that the fuzzy sets $\semantics{cats}$ and $\semantics{sleep}$ are defined as follows:
\vspace{-0.5em}
\begin{align*}\semantics{cats} = & \ 0.2 c_1 + 0.3 c_2 + 0.8 c_3 \\
\semantics{sleep} = & \ 0.5 c_1 + 0.4 c_2 + 0.4 c_3
\end{align*}
Then the value for ``several cats sleep" will be 
\begin{align*}
 \semantics{several}&\Big[\frac{\Sigma Count( 0.2 c_1 + 0.3 c_2 + 0.4 c_3)}{ 0.2 c_1 + 0.3 c_2 + 0.8 c_3}\Big]  \\
 =  \semantics{several}&\Big[\frac{0.9}{1.3}\Big]
 \end{align*}
Suppose  that the  possibility distribution $\semantics{\text{several}}$ will map low values to low values and very high values to low values, but intermediate values would be mapped to a high number as they still represent ``several". Thus the  proportion $\frac{9}{13}$, which  is  a high number, will evaluate to a high number. Thus the many valued relation of this statement will be high (a number close to 1). For examples of possibility distributions of some other fuzzy quantifiers, see \cite{Zadeh1983}.

%
%

\begin{proposition}
\label{prop:svqo}
The many valued semantics of a  sentence  with quantified object ``\emph{np v d np'}''  is the same as its  fuzzy quantifier semantics in $\VRel$.
\end{proposition}
\begin{proof}
After several steps of computation similar to those done in the proof of Proposition  \ref{prop:qsvo}, we obtain the following value for the semantics of ``\emph{np v d np'}\, ''
\[
\max_{(A,B,C)}
\min 
\Big(
\star \overline{\semantics{np}} A, 
\star \overline{\semantics{v}} (A, B),
\star \overline{\semantics{np'}} C, 
C \overline{\semantics{d}} B \cap C
\Big) 
\]
By applying Definition \ref{def:fuzzymodel} and  maximising the proportions,  the above unfolds to 
\[
\max_{(A,B,C)}
\min 
\left
(
\semantics{d}\Big[\frac{\Sigma Count(\semantics{v}[\semantics{np}] \cap \semantics{np'})}{\Sigma Count(\semantics{np'})}\Big] 
\right )\]
%
%
%
for $ \Pi_{Proportion(C|D) } = d$. Then the number computed above is the same as the one obtained from  $\semantics{d}\left[Proportion(\semantics{v}[\semantics{np}]|\semantics{np'})\right]$, which is the same as the  fuzzy quantifier semantics of ``d np' 's  are v-np's". 
\end{proof}
An example of this case is ``Mice eat several plants" which has the same semantics as ``Several plants are eaten by mice". Suppose we have fuzzy sets
\begin{align*}
\semantics{mice} &= \ 0.7 c_1 + 0.6 c_2 + 0.2 c_3 \\
\semantics{eat} &= \ 0.5 (c_1,c_1) + 0.8 (c_1,c_3) + 0.2 (c_2,c_1)  \\
\ &+ 0.3 (c_2,c_3) + 0.9 (c_3,c_3)\\
\semantics{plants} &= \ 0.2 c_1 + 0.3 c_2 + 0.6 c_3
\end{align*}
Then the semantics we get is
$$ \semantics{several}\Big[\frac{\Sigma Count (\semantics{eat}(\semantics{mice}) \cap \semantics{plants})}{\Sigma Count(\semantics{plants})}\Big] $$
The application of the verb to its subject gives
$$ \semantics{eat}(\semantics{mice}) = 0.5 c_1 + 0.7 c_3 $$
As a result,  the whole expression now evaluates to
\begin{align*}
 \semantics{several}&\Big[\frac{\Sigma Count (0.2 c_1 + 0.6 c_3)}{\Sigma Count(0.2 c_1 + 0.3 c_2 + 0.6 c_3)}\Big] \\
& = \semantics{several}\Big[\frac{0.8}{1.1}\Big]
 \end{align*}
 \noindent 
 This will yield another relatively high value for the many valued semantics of  this sentence,  as $Proportion(\semantics{eat}(\semantics{mice})|\semantics{plants})$ certainly indicates a case of ``several" mice eating plants.
 
\begin{corollary} 
\label{cor:qsvqo}
The many valued semantics of a sentence with a quantified subject and a quantified object ``\emph{d np v d' np'}'' is the same as its  fuzzy quantifier semantics in  $\VRel$.
\end{corollary}
\begin{proof}
The proof is obtained by applying propositions \ref{prop:qsvo} and \ref{prop:svqo}. After several steps of computation, we obtain that the many valued semantics of ``\emph{d np v d' np'}'' is
\[
\max_{(A,B),(C,D)}
\min 
\Big(
\star \overline{\semantics{np}} A, 
A \overline{\semantics{d}} A \cap B,
\star \overline{\semantics{v}} (B, D),
\star \overline{\semantics{np'}} C, 
C \overline{\semantics{d'}} C \cap D
\Big) 
\]
The above  unfolds and simplifies as before. When the maximum of the min set is realised, we obtain equivalence with the following
\[
\semantics{v}[\semantics{d}[\semantics{np}],\semantics{d'}[\semantics{np'}]]
\]
for $\Pi_{Proportion(\semantics{\tilde{v}}|\semantics{d'}[\semantics{np'}])} = d'$ \ and \ $\Pi_{Proportion(\semantics{v}[\semantics{np}]|\semantics{np})} = d$ and where $\semantics{\tilde{v}}$ is the image of $v$ on ``\emph{d np}", given by
\[
\semantics{\tilde{v}} = \semantics{v}[\semantics{d}(\semantics{np})]
\]
\end{proof}
An example of this case is  ``Several mice eat most plants" which has the same semantics as ``most plants are eaten by several mice". Given that  the fuzzy sets representing mice and plants are as before and taking the same fuzzy relation  for $\semantics{\text{eat}}$, we compute the meaning of this sentence. Suppose  further  that $\semantics{\text{most}}$ is a possibility distribution that assigns the value $0$ to numbers below $0.5$, and gradually increasing the value for numbers from $0.5$ to $1$. 

First, we compute the application of the quantifiers to their respective noun phrases:
\begin{align*} \semantics{several}[\semantics{mice}] = & \\
\operatorname*{arg\,max}_B \Big( \semantics{several} & \Big[ \frac{\Sigma Count(\semantics{mice} \cap B)}{\Sigma Count(\semantics{mice})} \Big] \Big) 
\end{align*}
\noindent 
If we assume that ``several" has the highest value for $0.4$, then it would for instance assign to the set $0.4 \semantics{mice}$ the value  $\Sigma_i 0.4 \mu_i u_i$ for $\mu_i u_i$ in $\semantics{mice}$. The second application gives
\begin{align*} \semantics{most}[\semantics{plants}] = \\
\operatorname*{arg\,max}_A \Big( \semantics{most} & \Big[ \frac{A \cap \semantics{plants}}{\Sigma Count(\semantics{plants})} \Big] \Big) 
\end{align*}
This will set $A = \semantics{plants}$, given that $1$ has the highest probability of being ``most".

The value of the whole sentence will be the verb applied to the quantified subject and object, hence we obtain
\begin{align*}
&\semantics{eat} \Big[\semantics{several}[\semantics{mice}], \semantics{most}[\semantics{plants}] \Big] \\
= &\semantics{eat} \Big[0.4 \semantics{mice}, \semantics{plants} \Big]\\
= &\max_{a,b} \min ( \mu_{0.4 \semantics{mice}}(a), \mu_{\semantics{eat}}(a,b), \mu_{\semantics{plants}}(b))\\
= &\max \Big( \min(0.28,0.5,0.2), \min(0.28, 0.8, 0.6), \\
&\min (0.24,0.2,0.2),  \min (0.24,0.3,0.6), \\
&\min (0.08,0.9,0.6) \Big)\\
=  &\max  (0.2, 0.28, 0.2,  0.24, 0.08)\\
= &0.28
\end{align*}
\noindent 
This means that  the extent to which several mice eat most plants is 28\%.

We conclude this section by   defining  the notion of a  degree of truth for sentences in $\VRel$ and noting that in the absolute case, it is the same as the truth value of the sentence in $\Rel$.

\begin{definition}
A quantified sentence $s$ has a degree of  truth $r$ iff $\overline{\semantics{s}} = r$ in $(\VRel, {\cal P}(U), \{\star\}, \overline{\semantics{\ }})$. 
\end{definition}

\begin{remark}
Suppose $ \overline{\semantics{s}}  = 1$ in $(\VRel, {\cal P}(U), \{\star\}, \overline{\semantics{\ }})$ and consider a sentence of the form ``d np v"; as proved in Proposition \ref{prop:qsvo} and by the above definition, this is the case iff we have $ \semantics{d}[Proportion(\semantics{vp}/\semantics{np})] = 1$\ for \ $\Pi_{Proportion(\semantics{vp} /\semantics{np})} = d$. This means that the proportion of elements of $\semantics{vp}$ that are in $\semantics{np}$ is $d$. Recall that our quantifiers are conservative, thus  the proportion of elements of $\semantics{vp} \cap \semantics{np}$ that are in $\semantics{np}$ is also $d$, which means $\semantics{vp} \cap \semantics{np} \in \semantics{d}[\semantics{np}]$, and according to Definition \ref{def:truth-genquant} this makes the sentence ``d np vp" have truth value \emph{true}.  In \cite{HedgesSadr}, authors  showed that this is equivalent  to $\overline{\semantics{s}} = t$ in $(\Rel, {\cal P}(U), \{\star\}, \overline{\semantics{\ }})$. The other direction holds in a similar fashion: suppose $\overline{\semantics{s}} = t$ in $(\Rel, {\cal P}(U), \{\star\}, \overline{\semantics{\ }})$, this is iff (as shown in \cite{HedgesSadr}),  the interpretation of ``d np vp"  is true in $\Rel$, which is iff $\semantics{vp} \cap \semantics{np} \in \semantics{d}[\semantics{np}]$, which by generalised quantifier theory means that $\semantics{vp} \cap \semantics{np} $ has $d$ elements of $\semantics{np}$, which will then make $ \semantics{d}[Proportion(\semantics{vp}/\semantics{np})]$ to be 1.  The case for sentences of the form ``np v d np" and ``d np v d' np" are similar. 
\end{remark}

\section{Conclusions and Future Work}
\label{sec:Intro}

In recent work \cite{HedgesSadr} showed how one can reason about generalised quantifiers using bialgebras over the category of sets and relations over a fixed powerset object (powerset of a universe of discourse).  They  developed an abstract categorical semantics and   instantiated it to  category of sets and relations.  Via the \emph{Set-to-Vector Space} and \emph{Relation-to-Linear Map} embedding, they transferred this semantics   from sets and relations to vectors and linear maps.  Their resulting  vectorial semantics, however,   is  hard to reason with and costly to implement, a fault  mainly due to the fact that in order to keep the  maps linear, they  had to work with vector spaces over powerset objects.

The reason for transferring the formal  semantics of natural language from sets and relations to  vector spaces and linear maps in compositional distributional semantics is to allow for quantitative reasoning in terms of the statistical data provided in  distributional semantics.  Another way to work with quantities and distributions  of data is to move to a fuzzy setting, as done in \cite{Zadeh1983}.  But then the question arises whether these two semantics are the same. This is the question to which  this  paper answers in positive. Concretely, in this paper we have shown  that the categorical version of fuzzy sets, that is category $\VRel$  of sets and many valued relations,  is compact closed  and defined over it the required  bialgebras. We  developed, within this category,  a many valued version of the abstract compact closed categorical semantics of \cite{HedgesSadr} with  Zadeh's fuzzy quantifiers  and showed  that the   two semantics amount to the same degrees of truth for quantified sentences.  As a result, in order to do quantification in compositional distributional semantics, one is not restricted to working with vector spaces over powerset objects and furthermore, fuzzy quantification is now also added to the existing setting.

A practical question that arises  is  what the   empirical statistical  consequences of embedding  $\FdVect$ in $\VRel$ are. In order to answer this question, we have  to work alongside  intuitions such as ``a distributional vector for a target word $w$ yields a fuzzy set whose degrees of membership are the degrees of co-occurrences of $w$ with a set of context words $c$,  or the degrees of similarity of  $w$  to $c$, or the degrees of contextual relevance of $w$ to $c$", or other similar readings.  Formally, one has to work alongside  the following diagram and   with category of matrices over reals; these are equivalent to $\VRel$ and a special case of category of $\FdVect$. 
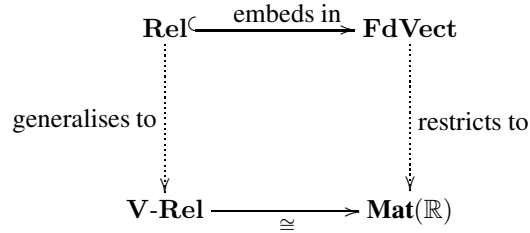
\begin{displaymath}
\xymatrix
{
\Rel  \ar@{^{(}->}[rr]^{\mbox{embeds in}} \ar@{.>}[dd]_{\mbox{generalises to}}   && \FdVect \ar@{.>}[dd]^{\mbox{restricts to}}  \\
&&\\
\VRel  \ar[rr]_{\cong} && \text{\bf Mat}(\mathbb{R})
}
\end{displaymath}
Building on the above ideas and implementing  our model on real data and experimenting with it constitutes future work.

Another future direction is to  build on the above intuitions and use the logic of fuzzy sets to develop  a logic for distributional data.  The quantified vectorial setting of  \cite{HedgesSadr} did  not allow for a natural notion of logic: the main vector space, which was spanned by a power set object,  did not have a natural interpretation  of  union and intersection of basis vectors in terms of  basic set theoretic operations. A fuzzy setting, however, gives rise to a fuzzy logic and provides semantics for  coordination in natural language, via operations such  fuzzy conjunction and disjunction, e.g. see \cite{Novak}.  Exploring the theoretical  corollaries of this fact and experimenting with them constitutes future work.

%

\section{Acknowledgments}

Support by EPSRC for Career Acceleration Fellowship  EP/J002607/2, and by AFOSR for International Scientific Collaboration Grant FA9550-14-1- 0079 is gratefully acknowledged by M.S. Support by a QMUL  PhD Scholarship is acknowledged by G.W. The authors are grateful for the fruitful comments of an anonymous reviewer. 

\bibliography{quant.bib}
\bibliographystyle{spphys}
\end{document}